\documentclass[accepted]{uai2022} 

\usepackage[american]{babel}

\usepackage{natbib} 
\usepackage{mathtools} 
\usepackage{booktabs} 
\usepackage{tikz} 

\usepackage{amsthm}
\usepackage{multirow}
\usepackage{csquotes}

\usepackage{graphicx}

\usepackage{algorithm}
\usepackage{algorithmic}

\newtheorem{define}{Definition}
\newtheorem{corollary}{Corollary}

\newtheorem{theorem}{Theorem}
\newtheorem*{theorem*}{Theorem}

\usepackage{amsmath,amssymb}
\usepackage{cleveref}

\usepackage[caption=false]{subfig}

\title{Estimating Categorical Counterfactuals via Deep Twin Networks}

\author[1,4]{\href{mailto:<av2514@ic.ac.uk>?Subject=Your Twin Nets paper}{Athanasios Vlontzos}}
\author[1,2]{Bernhard Kainz}
\author[3,4]{Ciar\'{a}n M. Gilligan-Lee}
\affil[1]{%
   Imperial College London,
    London, UK
}
\affil[2]{%
    FAU Erlangen-Nuremberg, 
    Erlangen, Germany
}
\affil[3]{%
    University College London,
    London, UK
  }

\affil[4]{%
    Spotify,
    London, UK
  }  
 
\begin{document}
\maketitle

\begin{abstract}
Counterfactual inference is a powerful tool, capable of solving challenging problems in high-profile sectors. To perform counterfactual inference, one requires knowledge of the underlying causal mechanisms. However, causal mechanisms cannot be uniquely determined from observations and interventions alone. This raises the question of how to choose the causal mechanisms so that resulting counterfactual inference is trustworthy in a given domain. This question has been addressed in causal models with binary variables, but the case of categorical variables remains unanswered. We address this challenge by introducing for causal models with categorical variables the notion of \emph{counterfactual ordering}, a principle that posits desirable properties causal mechanisms should posses, and prove that it is equivalent to specific functional constraints on the causal mechanisms. To learn causal mechanisms satisfying these constraints, and perform counterfactual inference with them, we introduce \emph{deep twin networks}. These are deep neural networks that, when trained, are capable of \emph{twin network} counterfactual inference---an alternative to the \emph{abduction, action, \& prediction} method. We empirically test our approach on diverse real-world and semi-synthetic data from medicine, epidemiology, and finance, reporting accurate estimation of counterfactual probabilities while demonstrating the issues that arise with counterfactual reasoning when counterfactual ordering is not enforced.\footnote{ Corresponding author: Athanasios Vlontzos athanasiosv@spotify.com; Further correspondence can be directed to Bernhard Kainz: b.kainz@imperial.ac.uk ; Ciar\'{a}n M. Gilligan-Lee: ciaranl@spotify.com}

\end{abstract}

\section{Introduction}


``If my credit score had been better, would I have been approved for this loan?'',  ``What is the effect of the diabetes type on the risk of stroke?''. Causal questions like these are routinely asked by scientists and the public alike. Recent machine learning advances have enabled the field to address causal questions in high-dimensional datasets to a certain extent  \cite{schwab2018perfect,alaa2017deep,shi2019adapting}.
However, most of these methods focus on {\tt Interventions}, which only constitute the second-level of Pearl's three-level causal hierarchy \cite{Pearl2009,bareinboim20201on}. At the top of the hierarchy sit {\tt Counterfactuals}. These subsume interventions and allow one to assign fully causal explanations to data. 

Counterfactuals investigate alternative outcomes had some pre-conditions been different. The crucial difference between counterfactuals and interventions is that the evidence the counterfactual is ``counter-to'' can contain the variables we wish to intervene on or predict. The first question posed at the start of this paper, for instance, is a counterfactual one. Here we want to know if improving our credit score will lead to loan approval in the explicit context that the loan has just been declined. A corresponding interventional query would be ``what is the impact of the credit score on the loan approval chances?''. Here, evidence that the loan has just been denied is \emph{not} used in estimating the impact. The second question posed at the start of this paper---regarding the effect of diabetes type on the risk of stroke---is an interventional question. By utilising this additional information, counterfactuals enable more nuanced and personalised reasoning and decision making. Counterfactual inference has been applied in high profile sectors like medicine \cite{richens2020improving, oberst2019counterfactual}, legal analysis \cite{lagnado2013causal}, fairness \cite{kusner2017counterfactual}, explainability \cite{galhotra2021explaining}, 
and  advertising \cite{ang2019unit}. 


To perform counterfactual inference, one requires knowledge of the causal mechanisms. However, the causal mechanisms cannot be uniquely determined from observations and interventions alone. Indeed, two causal models that have the same conditional and interventional distributions can disagree about certain  counterfactuals \cite{Pearl2009}. Hence, without additional constraints on 
the form of the causal mechanisms, 
they can generate ``non-intuitive" counterfactuals that conflict with domain knowledge, as originally pointed out by \cite{oberst2019counterfactual}. 

This raises the question of how best to choose the causal mechanisms so that resulting counterfactual inference is trustworthy in a given domain. Despite the importance of counterfactual inference, this question has only been addressed in causal models with binary treatment and outcome variables \cite{tian2000probabilities}. The case of categorical variables remains unanswered. Beyond binary variables, previous work has only derived upper and lower bounds for counterfactual probabilities \cite{Zhang2020BoundingCE}. In many cases, these bounds can be too wide to be informative.
We address this challenge by introducing for causal models with categorical variables the notion of \emph{counterfactual ordering}, a principle that posits desirable properties causal mechanisms should posses, and prove that it is equivalent to specific functional constraints on the causal mechanisms. Namely, we prove that causal mechanisms satisfying counterfactual ordering must be monotonic functions.

To learn such causal mechanisms, and perform counterfactual inference with them, we introduce \emph{deep twin networks}. These are deep neural networks that, when trained, are capable of \emph{twin network} counterfactual inference---an alternative to the \emph{abduction, action, \& prediction} method of counterfactual inference. Twin networks were introduced by \cite{balke1994counterfactual} and reduce estimating counterfactuals to performing Bayesian inference on a larger causal model, known as a \emph{twin network}, where the factual and counterfactual worlds are jointly graphically represented. Despite their potential importance, twin networks have not been widely investigated from a machine learning perspective. We show that the graphical nature of twin networks makes them particularly amenable to deep learning.

We empirically test our approach on a variety of real and semi-synthetic datasets from medicine and finance, showing our method achieves accurate estimation of counterfactual probabilities. Moreover, we demonstrate that if counterfactual ordering is not enforced, the model generates ``non-intuitive'' counterfactuals that contradict domain knowledge in these cases. Our contributions are as follows:  
\begin{enumerate}
\parskip0pt
    \item We introduce \emph{counterfactual ordering} for causal models with categorical variables, which posits desirable properties causal mechanisms should posses. 
    \item We prove \emph{counterfactual ordering} is equivalent to specific functional constraints on the causal mechanisms. Namely, that they must be monotonic. 
    \item We introduce \emph{deep twin networks} to learn such causal mechanisms and perform counterfactual inference. These are deep neural networks that, when trained, can perform \emph{twin network} counterfactual inference. 
    \item We test our approach on real and semi-synthetic data, achieving \emph{accurate} counterfactual estimation  that complies with domain knowledge.
\end{enumerate}


\section{Preliminaries}

\subsection{Structural causal models} \label{Section: Structural causal models}

While there are many paradigms to discuss causality, for example in explainable artificial intelligence (XAI), epidemiology, adversarial learning and econometrics, we work in the Structural Causal Models (SCM) framework. Through it, one is also able to derive the aforementioned field specific paradigms. As such in an effort to make our work more approachable we choose SCMs. Chapter $7$ of \cite{Pearl2009} gives an in-depth discussion and for an up-to-date review of counterfactual inference and Pearl's Causal Hierarchy, we invite the reader to refer \cite{bareinboim20201on}.

Furthermore, our work belongs to the large body of literature of causal inference. As such, certain assumptions about the availability of knowledge about causal links are made. Contrary to the field of causal discovery, the structural causal models are given and our task is develop methodologies regarding the inference and prediction of outcomes.

\begin{define}[Structural Causal Model] \label{functional causal model}
\label{scmdef}
A structural causal model (SCM) specifies a set of latent variables $U=\{u_1,\dots,u_n\}$ distributed as $P(U)$, a set of observable variables $=\{v_1,\dots, v_m\}$, a directed acyclic graph (DAG), called the \emph{causal structure} of the model, whose nodes are the variables $U\cup V$, a collection of functions $F=\{f_1,\dots, f_n\}$, such that $v_i = f_i(PA_i, u_i), \text{ for } i=1,\dots, n,$ where $PA$ denotes the parent observed nodes of an observed variable.
\end{define}


The collection of functions and distribution over latent variables induces a distribution over observable variables: $P(V=v) := \sum_{\{u_i \mid f_i(PA_i, u_i)\,=\,v_i\}} P(u_i).$ 
An example causal structure, represented as a directed acyclic graph (DAG), is depicted in Fig. \ref{examplescm}. We note that while this is a simplified causal structure our results extend to far more complex ones as those seen in fields like epidemiology.

\begin{define}[Submodel]
Let $M$ be a structural causal model, $X$ a subset of observed variables with realization $x$. A submodel $M_x$ is the causal model with the same latent and observed variables as $M$, but with functions replaced with
$F_{x} = \{f_i \mid v_i\notin X\} \cup \{f_{j}^{'}(PA_j, u_j) := x_j\ |\ v_j \in X\}$.
\end{define}


\begin{define}[do-operator]
Let $M$ be a structural causal model, $X$ a set of observed variables. The effect of action $\text{do}(X=x)$ on $M$ is given by the submodel $M_{x}$.
\end{define}

The $do$-operator forces variables to take certain values, regardless of the original causal mechanism. Graphically, $do(X=x)$ means deleting edges incoming to $X$ and setting $X=x$. Probabilities involving $do(x)$ are normal probabilities in submodel $M_x$: $P(Y=y \mid \text{do}(X=x)) = P_{M_x} (y)$.

\begin{figure}[t] 
\centering

\subfloat[
    \label{examplescm}]{
        \includegraphics[width=0.17\linewidth]{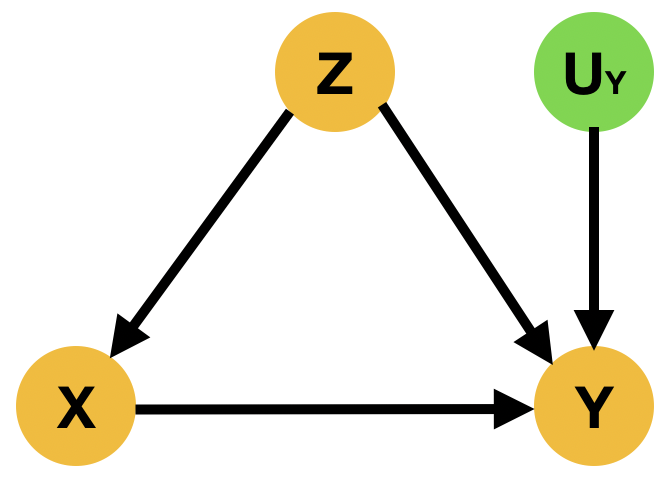}
        }
\subfloat[
    \label{examplescm_twin}]{
        \includegraphics[width=0.17\linewidth]{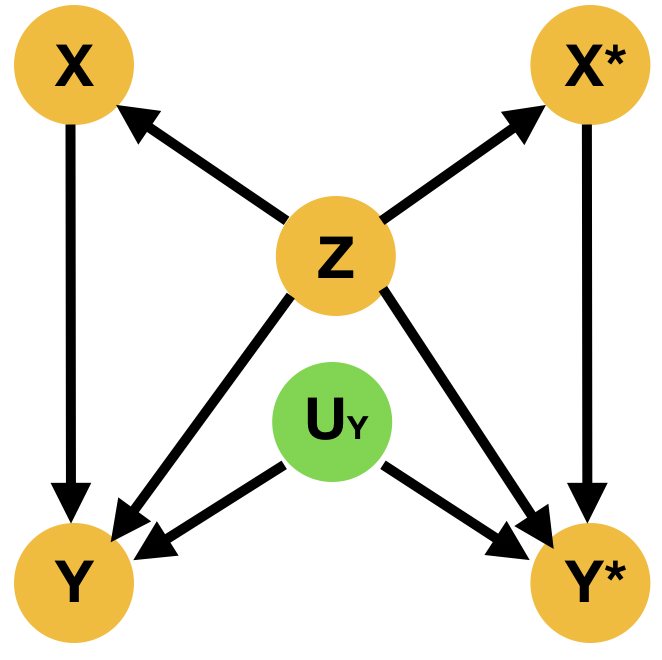}
        }
\subfloat[
    \label{examplescm_twin_do_x}]{
        \includegraphics[width=0.17\linewidth]{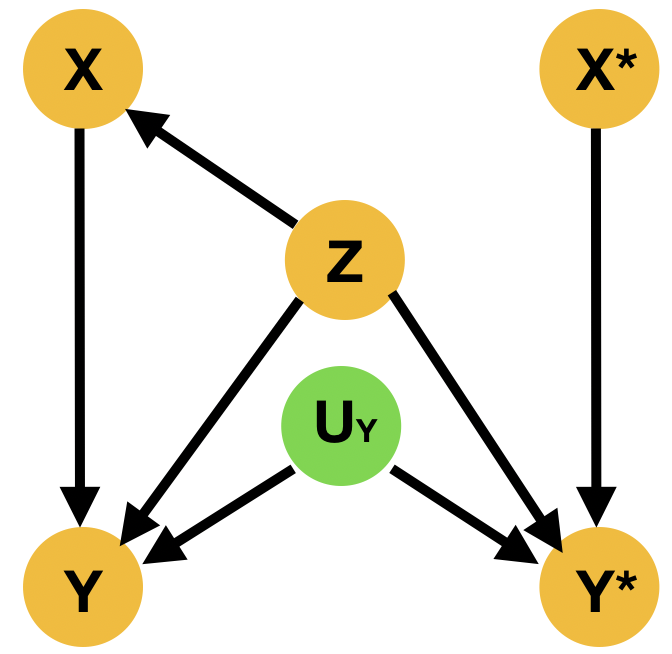}
        }
\subfloat[ 
    \label{examplescm_twin_do_x_2}]{
        \includegraphics[width=0.17\linewidth]{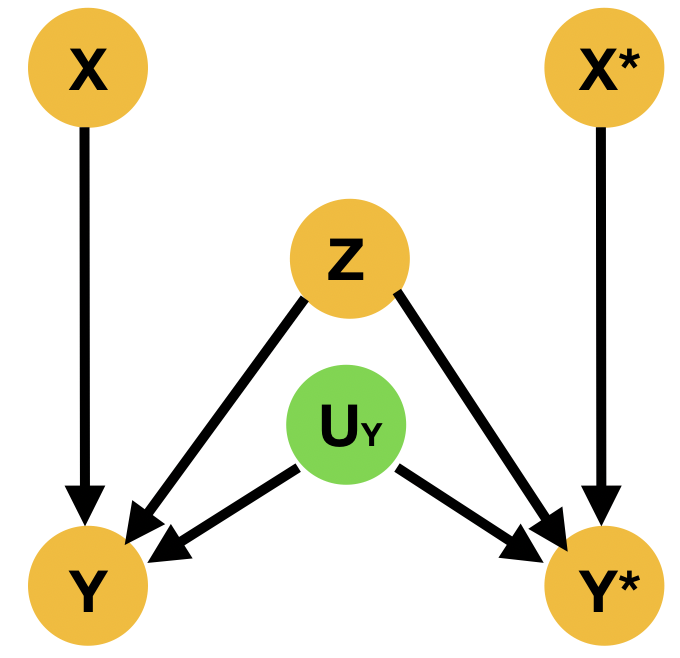}
        }
\subfloat[
    \label{figure: toy example}]{
        \includegraphics[width=0.17\linewidth]{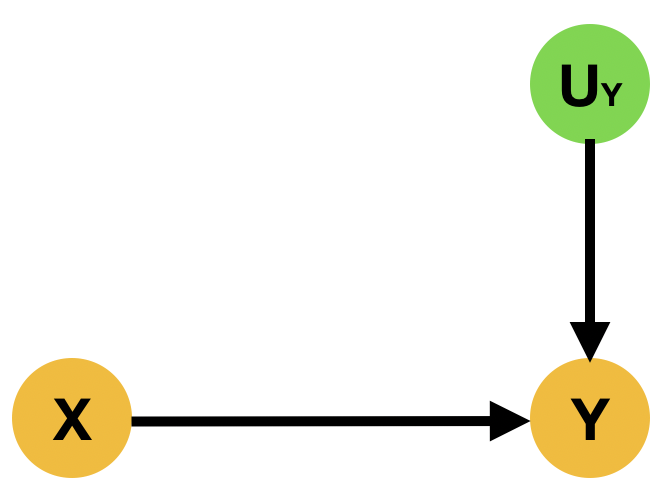}
        }
\caption{Orange nodes are observed, green latent. (a) Example SCM; (b) twin network of (a); (c) intervention in the twin network on node $X^*$; (d) interventions in the twin network on $X$ \& $X^*$; (e) Uncounfounded version of (a).
}
\label{figure: twin networks}
\end{figure}

\subsection{Counterfactual inference} \label{Section: two approaches to counterfactual inference}

\begin{define}[Counterfactual]\label{def:counterfactual}
The counterfactual ``$Y$ would be $y$ in situation $U=u$, had $X$ been $x$'', denoted $Y_{x}(u) = y$, equates to $Y=y$ in submodel $M_x$ for $U = u$.
\end{define}

The latent distribution $P(U)$ allows one to define probabilities of counterfactual queries, $P(Y_{y}=y) = \sum_{u \mid Y_{x}(u)=y} P(u).$ For $x \neq x'$ one can also define joint counterfactual probabilities,
$P(Y_{x}=y, Y_{x'}=y') = \sum_{u \mid Y_{x}(u)=y,\text{ }\& Y_{x'}(u)=y'} P(u).$
Moreover, one can define a counterfactual distribution given seemingly contradictory evidence. Given a set of observed evidence variables $E$, consider the probability $ P(Y_{{x}}={y}' \mid E = {e})$.
Despite the fact that this query may involve interventions that contradict the evidence, it is well-defined, as the intervention specifies a new submodel. Indeed, $ P(Y_{x}={y}' \mid E = {e})$ is given by ~\cite{Pearl2009}
$ \sum_{{u}}
P(Y_{{x}}({u})={y}')P({u}|{e})\,.$
The following theorem outlines how to compute such distributions.

\begin{theorem}[Theorem 7.1.7 in \cite{Pearl2009}] \label{counterfactual theorem}
Given SCM $M$ with latent distribution $P(U)$ and evidence ${e}$, the conditional probability $P(Y_{{x}} \mid {e})$ is evaluated as follows: 1) \textbf{Abduction:} Infer the posterior of the latent variables with evidence ${e}$ to obtain $P(U \mid {e})$, 2) \textbf{Action:} Apply $\text{do}({x})$ to obtain submodel $M_{x}$, 2) \textbf{Prediction:} Compute the probability of $Y$ in the submodel $M_{x}$ with $P(U \mid {e})$.

\end{theorem}

\subsection{Twin network counterfactual inference}
\label{section: twin networks}

A practical limitation of Theorem~\ref{counterfactual theorem} is that the abduction step requires large computational resources. Indeed, even if we start with a Markovian model in which background variables are mutually independent, conditioning on evidence---as in abduction---normally destroys this independence and makes it necessary to carry over a full description of the joint distribution over the background variables \cite{Pearl2009}. \cite{balke1994counterfactual} introduced a method to address this difficulty. Their method reduces estimating counterfactuals to performing Bayesian inference on an larger causal model, known as a \emph{twin network}, where the factual and counterfactual worlds are jointly graphically represented, described in technical detail below.  
We have included an extended discussion of the computational distinction between twin networks and abduction-action-prediction in the Appendix. 

A twin network consists of two interlinked networks, one representing the real world and the other the counterfactual world being queried. Constructing a twin network given a structural causal model and using it to compute a counterfactual query is as follows: First, one duplicates the given causal model, denoting nodes in the duplicated model via superscript $^*$. Let $V = \{v_1, \ldots , v_n\}$ be observable nodes in the causal model and $V^{*} = \{v_1^{*}, \ldots, v_n^{*}\}$ the duplication of these. Then, for every node $v_i^{*}$ in the duplicated, or ``counterfactual,'' model, its latent parent $u_i^{*}$ is replaced with the original latent parent $u_i$ in the original, or ``factual,'' model, such that the original latent variables are now a parent of two nodes, $v_i$ and $v_i^{*}$. The two graphs are linked only by common latent parents, but share the same node structure and generating mechanisms. To compute a general counterfactual query $P(Y=y \mid E=e, \text{do}(X=x))$, one modifies the structure of the counterfactual network by dropping arrows from parents of $X^*$ and setting them to value $X^*=x$. Then, in the twin network with this modified structure, one computes the following probability $P(Y^*=y \mid E=e, X^*=x)$ via standard inference techniques, where $E$ are factual nodes. That is, in a twin network one has:
\begin{equation} \label{equation: twin network}
\begin{aligned}
P(Y=y \mid &E=e, \text{do}(X=x))=\\
&P(Y^*=y \mid E=e, X^*=x)
\end{aligned}
\end{equation}

To illustrate this concretely, 
consider the causal model with causal structure depicted in \ref{examplescm}, where variables $X, Y$ are binary.  The counterfactual statement to be computed is 
$P\left(Y=0 \text{ }| \text{ } Y=1, \text{ do}(X=0)\right)$. 
The twin network approach to this problem first constructs the linked factual and counterfactual networks depicted in \ref{examplescm_twin}. The intervention $\text{do}(X^*=0)$ is then performed in the counterfactual network; all arrows from the parents of $X^*$ are removed and $X^*$ is set to the value $0$---graphically depicted in \ref{examplescm_twin_do_x}. The above counterfactual query is reduced to the following conditional probability in \ref{examplescm_twin_do_x}:
$ P\left(Y^*=0 \text{ }| \text{ } Y=1, X^*=0\right),$
which can be computed using Bayesian inference techniques. 

\textbf{Our contribution:} Despite their importance for counterfactual inference, twin networks have not been widely studied---particularly from a machine learning perspective. In the Methods section we demonstrate how to combine twin networks with neural networks to estimate counterfactuals.  

\subsection{Non-identifiability \& domain knowledge} 
Is non-identifiability of counterfactuals a problem? Given a causal model trained on observations and interventions, can we always trust its counterfactual predictions? In general the answer is no: counterfactual predictions from a causal model can conflict with domain knowledge---even if it perfectly reproduces observations and interventions, as we now show. 

In epidemiology, causal models with structure similar to the one of \Cref{figure: toy example} are studied, where $X$ is the presence of a risk factor and $Y$ is the presence of a disease. From epidemiological domain knowledge, it is believed that risk factors always increase the likelihood that a disease is present \cite{tian2000probabilities}---referred to as ``no-prevention'', that no individual in the population can be helped by exposure to the risk factor \cite{pearl1999probabilities}. Hence, if one observes a disease, but not the risk factor, then, in that context, if we had intervened to give that individual the risk factor, the likelihood of them not having the disease must be zero---as having the risk factor can only increase the likelihood of a disease. 

 Assume the simple case of DAG \ref{figure: toy example}, with $X,Y$ \emph{binary}, $U_Y$ a four-valued variable distributed under $q(U_Y)$ and $\neg X$ to be the logical negation of $X$.  

\begin{equation} \label{equation: example non-monotonic function}
Y=
    \begin{cases}
      X, & \text{if}\ U_Y=0 \\
      0, & \text{if}\ U_Y=1 \\
      1, & \text{if}\ U_Y=2 \\
      \neg X, & \text{if}\ U_Y=3 \\
    \end{cases}
\end{equation}

We'll now describe two causal models that generate the same observations and interventions, yet the counterfactuals generated by one model satisfy the above domain knowledge and the other do not. Consider the two different parameterizations of the causal model above: $\{q(U_Y)\}_0^3=\{1/2,1/6,1/6,1/6\}$ and $\{q(U_Y)\}_0^3=\{1/3, 1/3, 1/3, 0\}$. Both models have the same conditional distributions, and we have $P(Y_{X=1}=1) >  P(Y_{X=0}=1)$ and $P(Y_{X=1}=0) <  P(Y_{X=0}=0)$. This tells us that intervening to set $X=1$ \emph{always} makes $Y=1$ more likely, and doesn't increase the likelihood of $Y=0$ relative to $X=0$. So, at the interventional-level, these models seem to comply with Epidemiological domain knowledge that says that the presence of a risk factor, that is, $X=1$, always makes disease, $Y=1$, more likely.

Despite this, in the first model $P(Y_{X=1}=1 \mid Y=1, X=0) < P(Y_{X=0}=1 \mid Y=1, X=0)$. According to this model, $Y=1$ becomes \emph{less} likely when we intervene with $X=1$ in the counterfactual context $Y_{X=0}=1$---even though intervening to set $X=1$ can only make $Y=1$ more likely, and does not increase the likelihood of $Y=0$. This is a very ``non-intuitive'' counterfactual prediction from the point of view of an Epidemiologist.
However, in the second model, $P(Y_{X=1}=1 \mid Y=1, X=0) = P(Y_{X=0}=1 \mid Y=1, X=0)$. Indeed, in this model, no matter the counterfactual context, intervening to set $X=1$ \emph{never} reduces the likelihood $Y=1$. Thus the second model complies fully with Epidemiological domain knowledge.  


As the models agree on the data they're trained on, we must impose extra constraints to  learn the model that generates domain-trustworthy counterfactuals. In the next section, we present a  simple principle that provide such constraints.

\subsection{Related Works}

Despite the large body of work using machine learning to estimate interventional queries, which we discuss in detail in the Appendix, relatively little work has explored using machine learning to estimate counterfactual queries.

Recent work from \cite{pawlowski2020deep} used normalising flows and variational inference to compute counterfacual queries using abduction-action-prediction. A limitation of this work is that identifiability constraints required for the counterfactual queries to be uniquely defined given the training data are not imposed. Work by \cite{oberst2019counterfactual}, expanded by \cite{lorberbom2021learning}, used the Gumbel-Max trick to estimate counterfactuals, again using abduction-action-prediction. While this methodology satisfies generalisations of the monotonicity constraint, it does so because the Gumbel-Max trick has a limit on the type of conditional distributions it can generate---not because the authors imposed partial-identfiability constraints during the learning process. Hence the Gumbel-Max may not be suitable for the computation of counterfactual queries requiring different (partial-)identifiability constraints. Additional work by \cite{cuellar2020non} devised a non-parametric method to compute the Probability of Necessity using an influence-function-based estimator. This estimator was derived under the assumption of monotonicity. A limitation of this approach is that a separate estimator must be derived and trained for each counterfactual query. 

In the field of XAI \cite{joshi2019towards,pawelczyk2022exploring} used an autoencoder and GAN based approaches for computing counterfactuals, unfortunately without exploring the causal identifiability requirements that our work does.

While our architecture shares similarities with some of the above, there is one main difference. By interpreting our architecture as a twin network in the sense of \cite{balke1994counterfactual} and explicitly including an input for the latent noise term $U_Y$, we elevate our network from estimating interventional queries to counterfactual ones by performing Bayesian inference on the trained model.

A large body of work address the issue of \emph{partial} identifiability of counterfactuals from data. Historically, this line of work was initiated by \cite{balke1997bounds}, who explored bounds on the probabilities of causal queries of binary variables using linear programming. Recently, \cite{Zhang2020BoundingCE} extended such linear programming derived upper and lower bounds beyond binary outcomes to the case of continuous outcomes. Additional work by \cite{zhang2021partial} bounded counterfactuals by mapping the SCM space onto a new one that is discrete and easier to infer upon. Finally, \cite{10.2307/2951620} proposed Local average treatment effects as means of identifications of interventional queries from observational data.

\section{Methods}\label{section: methods}


\subsection{Counterfactual ordering} \label{counterfactual ordering}
We continue to consider the causal structure above with a DAG as depicted in \Cref{figure: toy example}. However, now $X, Y$ are categorical variables with an arbitrary number of categories $N, M$ each. 
Inspired by the Epidemiological example from the previous section, we now define \emph{counterfactual ordering}, which posits an intuitive  relationship between counterfactual and intervantional distributions. 


\begin{define}[Counterfactual Ordering]
A causal model with categorical treatment variable $X$ and categorical outcome variable $Y$ satisfies \emph{counterfactual ordering} if there exists an ordering on interventions and outcomes $\{x_0, x_1, \dots, x_N\}$, $\{y_0, y_1, \dots, y_M\}$ such that $P(Y_{x_i} = y_k) \geq P(Y_{x_j} = y_k)$ and $P(Y_{x_i} = y_h) \leq P(Y_{x_j} = y_h)$ for all $i>j$ and $k>h$, then it must be the case that $P(Y_{x_i} \geq y_k | Y_{x^*} = y^*) \geq P(Y_{x_j} \geq y_k | Y_{x^*} = y^* )$ 
for all counterfactual contexts $\{y^*, x^*\}$.
\end{define}

This encodes to the following intuition: If intervention $x_i$ only increases the likelihood of outcome $y_k$, relative to any intervention $x_j$ with $j<i$, without increasing the likelihood of $y_h$ for all $h<k$, then intervention $x_i$ must increase the likelihood that the outcome we observe is at least as high as $y_k$, regardless of the context. 
Counterfactual ordering places the following constraints on a causal model.


\begin{theorem} \label{theorem: counterfactual ordering}
If counterfactual ordering holds $P(Y_{x_j} = y_l | Y_{x_i} = y_h) = 0$ for all $l >h$ and $i > j$.
\end{theorem}

\begin{proof}
First note that $P(Y_{x'} = y' | Y_{x'} = y)=0$ for any $y' \neq y$ follows from the definition of counterfactuals. The conjunction of this and counterfactual ordering implies $0 = P(Y_{x_i} \geq y_k | Y_{x_i} = y_h) \geq P(Y_{x_j} \geq y_k | Y_{x_i} = y_h )$. As probabilities are bounded below by $0$, we have $P(Y_{x_j} \geq y_k | Y_{x_i} = y_h) = \sum_{l > h} P(Y_{x_j} = y_l | Y_{x_i} = y_h) = 0$.  Again, as probabilities are non-negative, we have $P(Y_{x_j} = y_l | Y_{x_i} = y_h) = 0$ for all $l >h$ and $i > j$.
\end{proof}



 Equality's of the form $P(Y_x = y' | Y_{x'} = y)=0$ are equivalent to the statement $\{Y_{x}=y'\} \wedge \{Y_{x'} = y\} = \text{False},$ where $\wedge$ is the logical AND operator. Therefore, the conjunction of the input-output pairs $X=x, Y=y'$ and $X=x, Y=y$ cannot occur in such a causal model. This yields constraints on the model parameters beyond those imposed by observations and interventions that can be enforced during causal model training. 

It is important to note that we are not saying every causal model should satisfy counterfactual ordering. As in all works on causal inference, it is ultimately up to the analyst to decide if such an assumption appears reasonable in a given domain. As discussed in the previous section, counterfactual ordering appears a reasonable assumption in Epidemiology. Moreover, are derived results relate to the cases where the interventions are on a single variable. While we do not constrain the number of variables with causal effect on the outcome, we note that multiple simultaneous treatments should be considered under the prism of disentangling treatment effects like in \cite{parbhoo2021ncore} In the Experiments section, we empirically demonstrate on data from medicine and finance that models that are trained to satisfy counterfactual ordering comply with domain knowledge, while models that aren't appear in conflict with domain knowledge.

\subsection{Counterfactual ordering and Counterfactual Stability}

\emph{Counterfactual stability} has been proposed by \cite{oberst2019counterfactual} as a different way to restrict the type of counterfactuals a causal model can output, to ensure they are ``intutive''. We define counterfactual stability then prove a relation between it and counterfactual ordering.

\begin{define}[Counterfactual Stability]
A causal model of categorical variable $Y$ satisfies
\emph{counterfactual stability} if it has the following property: If we observe $Y_x=y$, then for all $y' \neq y$, the condition $\frac{P(Y_x = y)}{P(Y_{x'}=y')} \geq \frac{P(Y_x=y')}{P(Y_{x'}=y)}$ implies that  $P(Y_x = y' | Y_{x'}=y) = 0$. That is,
if we observed $Y=y$ under intervention $X=x$, then the counterfactual outcome under
intervention $X=x'$ cannot be equal to $Y=y'$ unless the multiplicative change in $P(Y_x = y)$ is less than the multiplicative change in $P(Y_x = y')$.
\end{define}

This encodes the following intuition about counterfactuals: If we had taken an alternative action that would have only increased the probability of $Y=x$, without increasing the likelihood of other outcomes, then the same outcome would have occurred in the counterfactual case. Moreover, in order for the outcome to be different under the counterfactual distribution, the relative likelihood of an alternative outcome must have increased relative to that of the observed outcome.

Counterfactual stability is weaker than counterfactual ordering, as it imposes fewer constraints on the model. In fact,  counterfactual ordering between intervention and outcome values implies counterfactual stability holds between them:

\begin{theorem} \label{theorem: relation between counterfactual ordering and stability}
If a causal model satisfies \emph{counterfactual ordering} then it satisfies \emph{counterfactual stability}.
\end{theorem}

\begin{proof}
We need to show that when counterfactual ordering holds, $\frac{P(Y_x = y)}{P(Y_{x'}=y')} \geq \frac{P(Y_x=y')}{P(Y_{x'}=y)}$ $\implies$  $P(Y_x = y' | Y_{x'}=y) = 0.$ From counterfactual ordering we have $P(Y_x = y) \geq P(Y_{x'}=y)$ and $P(Y_x = y') \leq P(Y_{x'}=y').$ The latter implies $\frac{P(Y_x=y')}{P(Y_{x'}=y')} \leq 1$, which when combined with the former yields: $P(Y_x = y) \geq \frac{P(Y_x=y')}{P(Y_{x'}=y')} P(Y_{x'}=y).$ Additionally, \Cref{theorem: counterfactual ordering} says counterfactual ordering implies $P(Y_x = y' | Y_{x'}=y) = 0$, concluding the proof.
\end{proof}

\subsection{Counterfactual ordering functionally constrains causal mechanisms to be monotonic}


\cite{oberst2019counterfactual} were unable to derive any general functional constraints counterfactual stability places on the causal mechanisms underlying a given causal model. They were only able to compute counterfactuals satisfying it in a single, specific type of causal model. Namely, one where the mechanisms are parameterised using the Gumbel-Max trick. By contrast, we now derive general a functional constraint on causal mechanisms that is equivalent to counterfactual ordering. In the next section we will show how to learn causal models that satisfy this constraint. Thus we are able to learn causal models that satisfy counterfactual ordering without the need for specific parametric assumptions---such as the Gumbel-Max trick, as was required for counterfactual stability by \cite{oberst2019counterfactual}.


\begin{define}[Monotonicity]
If there exists an ordering on interventions and outcomes: $\{x_0, x_1, \dots, x_N\}$, $\{y_0, y_1, \dots, y_M\}$ such that $P(Y_{x_i} = y_k) \geq P(Y_{x_j} = y_k)$ and $P(Y_{x_i} = y_h) \leq P(Y_{x_j} = y_h)$ for all $i>j$ and $k>h$ then $Y_x(u) \geq Y_{x'}(u)$ for all $u$. Equivalently, the events $\{Y_{x_i}=y_h\} \wedge \{Y_{x_j}=y_l\}=\text{False}$ for all $i>j$ and $h<l$.
\end{define}

\begin{theorem} \label{theorem: relation between counterfactual ordering and monotonicity}
Given an intervention \& outcome ordering, counterfactual ordering \& monotonicity are equivalent.
\end{theorem}

\begin{proof}
First we show counterfactual ordering implies monotonicity. From Theorem~\ref{theorem: counterfactual ordering}, counterfactual ordering implies $\{Y_{x_i}=y_h\} \wedge \{Y_{x_j}=y_l\}=\text{False}$ for all $i>j$ and $h<l$, as $P(Y_{x_i}=y_h | Y_{x_j}=y_l)=0$. Monotonicity follows. 

Next we show monotonicity implies counterfactual ordering. Monotonicity implies that for intervention $X=x$, the likelihood of the outcome being higher in the outcome ordering increases, while the likelihood of the outcome being lower in the ordering decreases relative to the likelihoods imposed by intervention $X=x'$ which lies lower in the intervention ordering. 
All that remains is to show that for such interventions, $x, x'$, and outcome $Y=y$ for which $P(Y_x=y) \geq P(Y_{x'}=y)$, it follows that $P(Y_x \geq y | Y_{x^*} = y^*) \geq P(Y_{x'} \geq y | Y_{x^*} = y^* )$ 
for all counterfactual contexts $\{y^*, x^*\}$.  

$P(Y_x \geq y | Y_{x^*} = y^*)$ is computed by first updating $P(U)$ under $x^*,y^*$ and computing $P(Y \geq y)$ in the submodel $M_x$. That is, it corresponds to the expected value of $P(Y_x(u) \geq y)$ under $u\sim P(U|x^*,y^*)$. 
From monotonicity one has $Y_x(u) \geq Y_{x'}(u)$ for all $u$. Hence, for any $U=u$ that results in $Y_{x'}(u) \geq y$, that same $U=u$ yields $Y_{x}(u) \geq y$, as $Y_{x}(u) \geq Y_{x'}(u) \geq y$. Hence, as there are at most as many values of $U$ that lead to $Y_x \geq y$ as lead to $Y_{x'} \geq y$, one has $P(Y_x(u) \geq y) \geq  P(Y_{x'}(u) \geq y)$ for any $U=u$. This follows because $P(Y_x(u) \geq y) = \sum_{u|Y_x(u) \geq y} P(U=u)$ together with the observation that summands in $\sum_{u|Y_{x'}(u) \geq y} P(U=u)$ are a subset of summands in $\sum_{u|Y_x(u) \geq y} P(U=u)$. Taking expectations under $P(U|y^*, x^*)$ yields the proof.
\end{proof}

\cite{tian2000probabilities} proved that in an SCM with DAG \ref{examplescm} with binary $X, Y$, where $Y$ is monotonic in $X$, the probabilities of causation---important counterfactual queries that quantify the degree to which one event was a necessary or sufficient cause of another---can be uniquely identified from observational and interventional distributions. See \Cref{prob_of_causation_def} for a definition of the probabilities of causation. 
We thus have the follow corollary to theorem~\ref{theorem: relation between counterfactual ordering and monotonicity}.

\begin{corollary}\label{corollary}
In a counterfactually ordered SCM with DAG \ref{examplescm} and binary $X, Y$, the probabilities of causation are identified from observational and interventional distributions.
\end{corollary}

For categorical variables beyond the binary case, it is unknown whether monontonicity implies unique identifiability. However, in this work we are not concerned with counterfactuals being uniquely defined, as long as ``non-intuitive'' counterfactuals are ruled out. Constraints on the model beyond those imposed by observations and experimental data are said to \emph{partially-identify} counterfactual distributions. In the next section we demonstrate how to learn causal models satisfying counterfactual ordering from data.

\subsection{Deep twin networks}\label{Section: Combiningg twin networks and monotonicity with neural networks to estimate the probabilities of causation}


We now present \emph{deep twin networks} which combine twin networks with neural networks to learn the causal mechanisms and estimate counterfactuals. Importantly, we will discuss how to ensure the function space learned by the neural network satisfies counterfactual ordering.

Contrary to prior Bayesian network approaches deep twin networks allow us not only to estimate counterfactual probabilities from data but to learn the underlying functions that dictate the interactions between causal variables. Hence we are able to gain a deeper insight on the mechanisms represented in the structural causal model that generate the counterfactuals we want to predict. Moreover we are able to quantify the uncertainty about the outcome by learning the latent noise distribution. Finally, the use of neural networks allow us to gain flexibility and computational advantages not present in previous plug-in estimators. Specifically we will see that our methods admits an arbitrary number and type of confounders $Z$ while estimating counterfactual probabilities, a stark difference to plug-in estimators, such as those presented in \cite{cuellar2020non}. 

Our approach has two stages, training the neural network such that it learns the counterfactually ordered causal mechanisms that best fit the data, then interpreting it as a twin network on which standard inference is performed to estimate counterfactual distributions. Note that if one can generate counterfactuals, one can also generate interventions. 

For clarity, we confine our explanations to the causal structure from \Cref{examplescm} where $X,Y$ are categorical variables with $X \in \{1, \dots, N\}$ and $Y\in \{1,\dots, M\}$, and $Z$ can be categorical or numerical. Note there can be many $Z$. Our method can be extended to multiple causes and a single output straightforwardly. To generalize this approach to an arbitrary causal structure, one applies our method to each parent-child structure recursively in the topological specified by the direction of the arrows in the causal structure. 

\noindent\textbf{Training deep twin networks:}
To determine the architecture of our neural network, we start with the causal structure of the SCM we wish to learn, and consider the graphical structure of its twin network representation. Our neural network architecture then exactly follows this graphical structure. This is graphically illustrated for the case of binary $X,Y$ from \Cref{examplescm} with twin network in \Cref{examplescm_twin_do_x} in \Cref{figure: deep twin netork architectures}. In the case of binary $X,Y$, the neural network has two heads, one for the outcome under the factual treatment and the other for the outcome under the counterfactual treatment. Furthermore two shared---but independent of one another---base representations, one corresponding to a representation of the observed confounders, $Z$, and the other to the latent noise term on the outcome, $U_Y$, are employed. For multiple treatments we have $N$ neural network heads, each corresponding to the categories of $X$. To interpret this as a twin network for given evidence $X,Y,Z$ and desired intervention $X^*$, we marginalize out the heads indexed by the elements of $\{1,\dots,N\} / X,X^*$. To train this neural network, we require two things: 1) a label for head $Y^*$, and 2) a way to learn the distribution of the latent noise term $U_Y$.

For 1), we must ask what the expected value of $Y^*$ is, for fixed covariates $Z$, under a change in input $X^*$. This corresponds to $\mathbb{E}(Y^* | X^*, Z)$. Given the correspondence between twin networks and the original SCM outlined in \Cref{equation: twin network} from \Cref{section: twin networks}, this corresponds to $\mathbb{E}(Y | do(X), Z)$, which is the expected value of $Y$ under an intervention on $X$ for fixed $Z$. There are many approaches to estimating this quantity in the literature \cite{shalit2016estimating, alaa2017deep,johansson2016learning, shi2019adapting}. We follow~\cite{schwab2018perfect}
due to their methods simplicity and empirical high performance. Any method that computes $\mathbb{E}(Y | do(X), Z)$ can be used, however. In addition to specifying the causal structure, the following standard assumptions are needed to estimate $\mathbb{E}(Y | do(X), Z)$ \cite{schwab2018perfect}: 1) \emph{Ignorability:} there are no unmeasured confounders; 2) \emph{Overlap:} every unit has non-zero probability of receiving all treatments given their observed covariates. Computing this expectation provides the labels for $Y^*$.

\begin{figure}[t]
    \centering
   
    \label{figure: deep twin netork architectures_TN}{
        \includegraphics[width=0.95\linewidth]{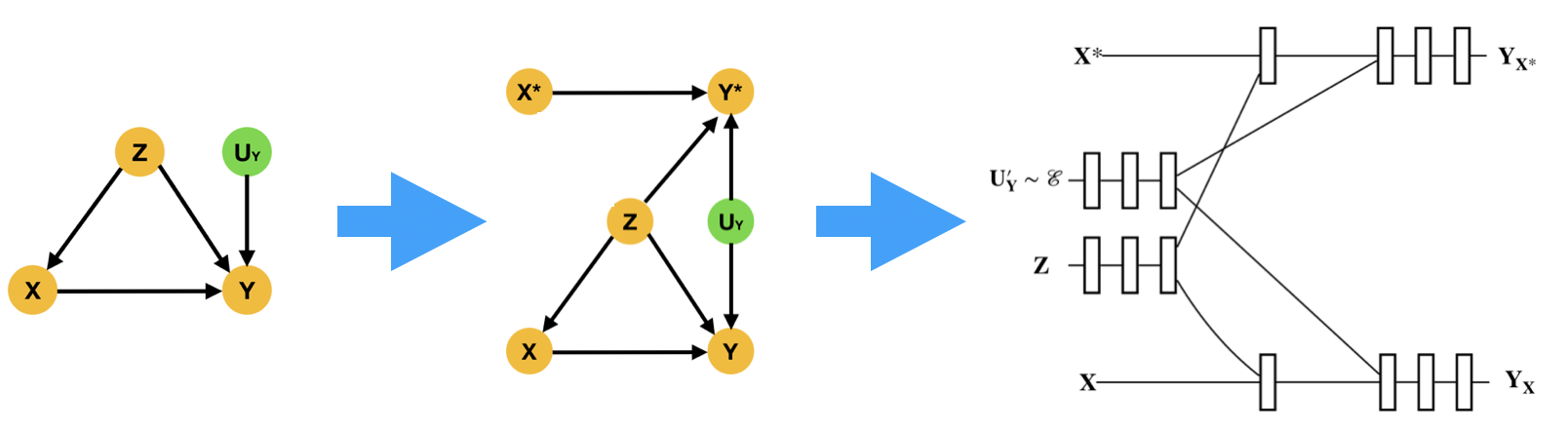}
        }
        
        
    \caption{From DAG to twin network DAG to deep neural network (NN) architecture for binary $X,Y$. Rectangular blocks are NN blocks, like FCN layers or Lattices; forward intersections are concatenation of features.}
    \label{figure: deep twin netork architectures}
\end{figure}





For 2), consider the following. Formally, the causal structure of \Cref{examplescm} has $Y=f(X,Z,U_Y)$ with $U_Y\sim q(U_Y)$ for some $q$. Without loss of generality \cite{goudet2018learning}, one can rewrite this as $Y=f(X,Z,g(U_Y'))$ with $U_Y'\sim \mathcal{E}$ and $U_Y = g(U_Y')$, where $\mathcal{E}$ is some easy-to-sample-from distribution, such as a Gaussian or Uniform. Hence we have reduced learning $q(U_Y)$ to learning function $g$, whose input corresponds to samples from a Guassian or Uniform.
Taken together, this provides a method to train our deep twin network. A summary is provided in \Cref{algorithm: training}. Key implementation details to note are that we are (a) using an MSE objective function casting the problem onto regression, as this was in practice easier to train; (b) the neural network employed two linear layers for each module representing the structure of the twin network, with 32 features each. 


\begin{algorithm}[h]
\caption{Training a deep twin network} \label{algorithm: training}
\scriptsize
\textbf{Input:} $X$: Treatment, $Z$: Confounders, $X^{*}$: Counterfactual Treatment; $Y$: Outcome; C: DAG of causal structure; I: loss imposing constraint on causal mechanisms
\\
\textbf{Output:} F: trained deep twin network
\begin{algorithmic}[1]
\STATE Set F's architecture to match twin network representation of C, as in Figure~\ref{figure: deep twin netork architectures}
\STATE To obtain label for counterfactual head, estimate $\mathbb{E}(Y | do(X), Z)$, yielding training dataset $\mathcal{D}:=\{X,X^*,Z;Y,Y^*\}$
\FOR {$x, x^{*}, z; y, y^* \in \mathcal{D}$ and $u_y \sim \mathcal{N}(0,1)$}
\STATE $y',y^{'*} = F(x,x^{*},u_y,z)$
\STATE Train F by minimizing $MSE(y,y') + MSE(y^{*},y^{'*}) + I(\mathcal{D})$ 
\ENDFOR
\end{algorithmic}
\end{algorithm}



\noindent\textbf{Enforcing constraints on the causal mechanisms:\label{}}
There are a few approaches to ensure that the function space learned by a neural network satisfies counterfactual ordering. Recall from Theorem~\ref{theorem: counterfactual ordering} that such constraints correspond to limits on the type of input-output pairs consistent with the function. 
One approach is to specify a loss function penalising the network for outputs that violate the constraints, as done in \cite{sill1997monotonicity}. Alternatively, 
counterexample-guided learning \cite{sivaraman2020counterexample} can be employed, to ensure the trained network does not produce any of these outputs when given the corresponding input. 
Lastly, as Theorem~\ref{theorem: relation between counterfactual ordering and monotonicity} equates counterfactual ordering with monotonicity, a recent method uses ``look-up tables'' \cite{gupta2016monotonic} to enforce monotonicity, and has been implemented in TensorFlow Lattice. 
Theorem \ref{theorem: relation between counterfactual ordering and monotonicity} requires that the treatment and outcome categorical variables are themselves ordered. One method of ordering our treatment variables is based on their perceived preference, i.e. based on domain knowledge. In the absence of domain knowledge one could look at the Average Treatment Effect (ATE)  $=:\frac{1}{N} \sum_i y_1(i)-y_0(i)$, where $y_1$ represents the treated outcome while $y_0$ the control outcome, as well as the interventional probabilities $P(Y\mid do(X))$. We then determine an estimate of the relationship governing the treatment and the outcomes by observing the trend of the ATE or $P(Y\mid do(X))$ as we change the treatments. The treatments are then ordered such that their relationship towards the outcomes can be characterized as monotonic. If one does not have access to interventional probabilities, then one can estimate them from observational data \cite{schwab2018perfect}. 


\paragraph{Estimating counterfactuals:} We can now use the trained model to perform counterfactual inference. The reason our neural network architecture matches the Twin Network structure is that performing Bayesian inference on the neural network explicitly equates to performing counterfactual inference. For \Cref{examplescm}, there are two counterfactual queries one can ask: (1) $P(Y_{X=x'}=y' \mid X=x, Y=y, Z=z)$, (2) $P(Y_{X=x}=y, Y_{X=x'}=y' \mid Z=z)$, where any of $x, y, z$ can be the empty set. Recall from Section~\ref{section: twin networks} that in a twin network (1) corresponds to $P(Y^*=y' \mid X=x, Y=y, X^*=x')$, and (2) corresponds  to $P(Y=y, Y^*=y' \mid X=x, X^*=x')$. Any method of Bayesian inference can then be employed to compute these probabilities, such as Importance or Rejection Sampling, or Variational methods. See Algorithm~\ref{algorithm: inference}. 

\begin{algorithm}[h]
\caption{Deep twin network counterfactual inference}\label{algorithm: inference}
\scriptsize
\textbf{Input:} $X$: Treatment, $U_Y$: Noise, $Z$: Confounders, $X^{*}$: Counterfactual Treatment; $Y$: Outcome $Y^{*}$: Counterfactual Outcome; F : Trained deep twin network; Q: desired counterfactual query (in this example, $Y_{X=x'}=y' \mid X=x, Y=y, Z=z$))
\\
\textbf{Output:} P(Q): Estimated distribution of Q.
\begin{algorithmic}[1]
\STATE Convert P(Q) to twin network distribution: $P(Y_{X=x'}=y' \mid X=x, Y=y, Z=z) \rightarrow P(Y^*=y' \mid X=x, Y=y, X^*=x')$ 
\STATE Compute $P(Y^*=y' \mid X=x, Y=y, X^*=x'):$ 
\FOR  {$x, x', z \in \mathcal{D}_{test}$}
 \FOR{$   [u_y]_N \sim \mathcal{N}(0,1), N\in \mathbb{N}$}
    \STATE Sample ($ \tilde{y}, \tilde{y}^* = F(x,x',u_y,z)$) such that $\tilde{y}=y$
    \STATE The frequency of these samples for which $\tilde{y}^*=y'$ yields P(Q)
 \ENDFOR
 \ENDFOR
\end{algorithmic}
\end{algorithm}

\section{Experimentation}

We now evaluate our \emph{counterfactual ordering} principle and our \emph{deep twin network} computational tool. 
We focus on four publicly available real-world datasets, the German Credit Dataset \cite{Dua:2019}, the International Stroke Trial (IST) \cite{IST}, the Kenyan Water task \cite{cuellar2020non} and the Twin mortality dataset~\cite{louizos2017causal}. We further use two synthetic and two semi-synthetic tasks to test our proposed methods. Full dataset description is in \Cref{appe_data_description}. 

We break our research questions (RQ) into two distinct types: ones that assess our counterfactual ordering principle, and ones that assess the counterfactual estimation accuracy of deep twin networks. Specifically, to asses counterfactual ordering, we wish to determine if not enforcing it leads to domain knowledge conflicting counterfactuals in real datasets, relative to enforcing it. To estimate counterfactual estimation accuracy of deep twin networks satisfying counterfactual ordering, we test on synthetic and semi-sythentic datasets---where we by design have access to the ground truth---as well as on a dataset involving twins, where we use features of one twin as a the counterfactual for the other. Finally, we also test on a real dataset involving binary treatment and outcome. We do this as Corollary~\ref{corollary} showed that in counterfactually ordered causal models with binary variables, certain counterfactual probabilities are uniquely identified from data. Hence for binary variables we can determine how accurate our deep twin network method is relative to these known identified expressions. 

\begin{itemize}
\item \textbf{RQ1:} If counterfactual ordering isn't enforced, do counterfactuals conflict with domain knowledge? 
\item \textbf{RQ2:} By imposing counterfactual ordering, do generated counterfactuals comply with domain knowledge?
\item\textbf{RQ3:} Can we accurately estimate counterfactual probabilities using deep twin networks?
 \end{itemize}
 
 \begin{table}[t]
\centering
\resizebox{1\columnwidth}{!}{
\begin{tabular}{rllll}
\toprule
  P(T',T) & 0.0                  & 1.0                  & 2.0                  & 3.0                  \\
\midrule
    0 & 0                    & $0.0816 \pm 0.1414 $ & $0.0860 \pm 0.1191 $ & $0.0344 \pm 0.0208 $ \\
    1 & $0.1439 \pm 0.1396 $ & 0                    & $0.1156 \pm 0.0984 $ & $0.1297 \pm 0.1306 $ \\
    2 & $0.1286 \pm 0.0726 $ & $0.1290 \pm 0.1347 $ & 0                    & $0.0741 \pm 0.0752 $ \\
    3 & $0.0680 \pm 0.0457 $ & $0.1854 \pm 0.1452 $ & $0.0974 \pm 0.1140 $ & 0                    \\
\bottomrule
\end{tabular}}
\caption{\textbf{Non-constrained model}. $P(T',T) = P(\text{Risk}_{\text{Account Status}=T'}=\text{good} \mid \text{Account Status}=T, \text{Risk}=\text{bad})$. Columns and rows are Treatments. 
We observe counter-intuitive probabilities as the lower triangular sub-matrix offers higher probabilities than the upper triangular one. That is, if we observe evidence where bad account status led to bad risk, the non-constrained model predicts an increase in net worth would have led to a \emph{lower} chance of being deemed a good risk---even though all other factors are kept fixed. An un-intuitive result that conflicts with domain knowledge of the finance industry.}\label{pn_account_risk_no_constr}
\end{table}

\begin{table}[t]
\centering
\resizebox{1\columnwidth}{!}{
\begin{tabular}{rllll}
\toprule
    P(T',T) & 0.0                  & 1.0                  & 2.0                  & 3.0                  \\
\midrule
    0 & 0                    & $0.3022 \pm 0.0415 $ & $0.3977 \pm 0.0382 $ & $0.4040 \pm 0.0381 $ \\
    1 & $0.1079 \pm 0.0322 $ & 0                    & $0.3891 \pm 0.0988 $ & $0.4118 \pm 0.0545 $ \\
    2 & $0.0670 \pm 0.0156 $ & $0.2816 \pm 0.0653 $ & 0                    & $0.4470 \pm 0.0751 $ \\
    3 & $0.1383 \pm 0.0442 $ & $0.2953 \pm 0.0344 $ & $0.3522 \pm 0.0577 $ & 0                    \\
\bottomrule
\end{tabular}}
\caption{\textbf{Counterfactual Ordering}. $P(T',T) = P(\text{Risk}_{\text{Account Status}=T'}=\text{good} \mid \text{Account Status}=T, \text{Risk}=\text{bad})$. Columns and rows are Treatments. We observe intuitive results as the  lower triangular sub-matrix offers lower probabilities than the upper triangular one. That is, when we observe evidence in which bad account status led to bad risk, the counterfactually ordered model predicts an increase in net worth would have led to a higher chance of being deemed a good risk---an intuitive result that complies with domain knowledge in the finance industry.}\label{pn_account_risk_count_ordering}
\end{table}

\subsection{Answering \textbf{RQ1} \& \textbf{RQ2}}

We investigate the German Credit real-world dataset and explore the International Stroke Trial dataset in the Appendix. We train a deep twin network on German Credit data using algorithm~1. In the Appendix we outline how we determined the monotonicity direction. 
In Table~\ref{pn_account_risk_count_ordering} and Table~\ref{pn_account_risk_no_constr} we estimate the counterfactual probability $P(\text{Risk}_{\text{Account Status}=T'}=\text{good} \mid \text{Account Status}=T, \text{Risk}=\text{bad})$ for a model satisfying counterfactual ordering and an unconstrained model respectively. 
That is, we ask what the probability that our loan risk would be good if we improved our account status, given that our account status is currently bad and we were just deemed a poor risk of a loan.  We note that the unconstrained model offers us non-intuitive probabilities that, when put in context, do not make sense in the real world. We observe that when we condition on evidence in which bad account status led to bad risk, the unconstrained model predicts that increasing an individuals net worth would have resulted in a lower probability of being deemed a good risk than \emph{decreasing} their net worth. This result defies common sense, answering RQ1. On the other hand, when we observe bad account status led to bad risk, the counterfactually ordered model predicts that an increase in an individuals net worth would have led to a higher chance of them being deemed a good risk than a decrease in their worth in this context. This result fits with our understanding of the financial industry, answering RQ2. 

We observe the same ``intuitive'' versus ``non-intuitive'' behavior for the Heparin treatment from the International Stroke Trial dataset in Appendix~\ref{appendix:unconstrained_ps}.

\begin{table*}[t]
\centering
\resizebox{\textwidth}{!}{
\begin{tabular}{@{}lll|ll|ll@{}}
                        & \multicolumn{2}{c|}{Credit Dataset} & \multicolumn{2}{c|}{IST - Aspirin} & \multicolumn{2}{c}{IST - Heparin} \\ \midrule
\multicolumn{1}{c}{\textbf{F1 Scores}} &
  \multicolumn{1}{c}{\begin{tabular}[c]{@{}c@{}}No Constrains\\ Linear Layers\end{tabular}} &
  \multicolumn{1}{c|}{\begin{tabular}[c]{@{}c@{}}Counterfactual \\ Ordering\end{tabular}} &
  \multicolumn{1}{c}{\begin{tabular}[c]{@{}c@{}}No Constrains\\ Linear Layers\end{tabular}} &
  \multicolumn{1}{c|}{\begin{tabular}[c]{@{}c@{}}Counterfactual \\ Ordering\end{tabular}} &
  \multicolumn{1}{c}{\begin{tabular}[c]{@{}c@{}}No Constrains\\ Linear Layers\end{tabular}} &
  \multicolumn{1}{c}{\begin{tabular}[c]{@{}c@{}}Counterfactual \\ Ordering\end{tabular}} \\ \midrule
\textbf{Factual}        & 0.4929           & 0.8637           & 0.6113           & 0.6417          & 0.3497           & 0.9758          \\
\textbf{Counterfactual} & 0.4698           & 0.9795           & 0.7152           & 0.9501          & 0.4103           & 0.9851          \\ \bottomrule
\end{tabular}}%
\caption{F1 score of counterfactual predictions for semi-synthetic German Credit Dataset with Treatment: Existing account status, Outcome: Synthetic; \& International Stroke Trial (IST) Dataset with Treatment: Aspirin, Outcome: Synthetic; Treatment: Heparin,  Outcome: Synthetic. See the Appendix for dataset description.}
\label{f1_account_ist_synthetic}
\end{table*}

\subsection{Answering RQ3}
\paragraph{Synthetic data}
We first evaluate whether we can accurately estimate the probabilities of causation---defined in the Appendix---on synthetic data. We test on data generated by an unconfounded as well as a confounded synthetic causal model, whose functional forms are outlined in the Appendix. Following the algorithms~\ref{algorithm: training},~\ref{algorithm: inference} we train a deep twin network on data from each case and enforce monotonicity. In both cases, we show accurate estimation. Results are in Table~3 and Fig.~3 in the Appendix.

\paragraph{Semi-synthetic data}
In Table \ref{f1_account_ist_synthetic} we show the results of our semi-synthetic experiments, described in the Appendix, for both the German Credit Datasest as well as the International Stroke Trial. Here, as the outcomes are synthetic, the ground truth is known a priori, hence we are able to calculate the associated F1 scores for each of the models. The counterfactually ordered models are more accurate at predicting both factual and counterfactual outcomes answering RQ3. Moreover, not enforcing counterfactual ordering leads to reduced performance in counterfactual estimation.



\paragraph{Real-world data}
We show 
performance on the Twin Mortality data of \cite{louizos2017causal} now, and discuss the Kenyan Water task from \cite{cuellar2020non} in the Appendix. In the case of Kenyan Water dataset, treatment ad outcome are binary, so we can compare the deep twin networks estimated counterfactual distributions to the uniquely identified counterfactual distribution via Corollary~\ref{corollary}. We report accurate estimation in Table~\ref{results_table_2} in the Appendix. 

In the Twin mortality dataset the goal is to understand 
the effect being born the heavier of the twins has on mortality one year after birth, given confounders regarding the health of the mother and background of the parents. 
Previous work addressed this with intervention queries. We use counterfactual queries---specifically the probabilities of causation. We follow~\cite{louizos2017causal,yoon2018ganite}'s preprocessing. As in \cite{louizos2017causal}, we treat each twin as the counterfactual of their sibling---providing a ground truth reported in \ref{results_table_2} in the Appendix. Again, monotonicity is justified here as we do not expect increasing birth weight to lead to reduced mortality.

First, given birth weight and mortality evidence provided by one twin, we aim to estimate the expected counterfactual outcome had their weight been different. That is, compute $\mathbb{E}(\text{Mortality}_\text{Weight} | \text{Mortality}^*, \text{Weight}^*, Z)$, where $Z$ are observed confounders. We achieve a counterfactual AUC-ROC of $86\%$ and F1 score of $83\%$. 
\cite{louizos2017causal} addressed this same question using only used interventional queries. That is, they computed $\mathbb{E}(\text{Mortality}_\text{Weight} | Z)$ and only achieve AUC $83\%$. We thus outperform \cite{louizos2017causal}'s AUC by $3\%$. Full results in Appendix Table \ref{results_table_2}. Here, by explicitly conditioning on and using the fact that the observed twins had birth weight and mortality, we are able to update our knowledge about the latent noise term of the other twin. Our improved AUC score showed using this allowed more accurate estimation of the ``hidden'' twins outcome. This cleanly illustrate the difference between interventions and counterfactuals. To give a comparison to prior work, we computed the average treatment effect from our model, yielding $-2.34\% \pm 0.019$ which matches \cite{louizos2017causal}---showing our model accurately estimates interventions as well as counterfactuals.

Table~\ref{results_table_2} reports our estimation of the Probabilities of Causation. Note that no previous work has computed these counterfactual distributions. Despite accurate estimation of the Probability of Sufficiency, and Necessity \& Sufficiency, our model underestimates the Probability of Necessity. This can be explained by a large data imbalance regarding the mortality outcome
---affecting the Probability of Necessity the most as mortality is the evidence conditioned here. Nevertheless, we correctly reproduce the relative sizes of the Probabilties of Causation, with Probability of Necessity an order of magnitude larger than the others.

\section{Conclusions}
We motivated and introduced \emph{counterfactual ordering}, a principle that posits desirable properties causal mechanisms should posses. We proved it is equivalent to causal mechanisms being monotonic. To learn such mechanisms, and perform counterfactual inference with them, we introduced \emph{deep twin networks}. We empirically tested our approach on real and semi-synthetic data, achieving \emph{accurate} counterfactual estimation that complies with domain knowledge. 

\bibliography{bibliography.bib}

\section{Acknowledgments}
We would like to acknowledge and thank our sources of funding and support for this paper. Funding for this work was received by Imperial College London and the  MAVEHA (EP/S013687/1) project and the UKRI London Medical Imaging \& Artificial Intelligence Centre for Value-Based Healthcare (A.V., B.K.). The authors also received GPU donations from NVIDIA. 

\section{Author Contributions}
A.V. and C.L. contributed in the theoretical formulations; A.V. developed the codebase and run the experiments; A.V, B.K. and C.L. contributed to the manuscript. 

\section{Competing Interests}
The authors declare no competing interests.

\paragraph{Data Availability}
All our datasets are publicly available and free to use for research purposes. The Kenyan water dataset originates from \cite{DVN28063_2015} licensed under a non commercial use clause and with the requirement for secure storage, both conditions have been fulfilled by the authors. The Twin Mortality dataset on the other hand was used as supplied by \cite{yoon2018ganite}. Finally the semi-synthetic and synthetic datasets can be replicated with the code provided

\paragraph{Code Availability}
Our codebase is available in \cite{Vlontzos2022} for public use under an MIT license.

\newpage

\newpage

\section*{Appendix}
\appendix

\section{Supplementary Results}

\subsection{Probabilities of Causation: Definitions}\label{prob_of_causation_def}
In our main text we assumed familiarity with the concept of the probabilities of causation, in order to make this work more accessible to those not familiar with the mathematical definition of the probabilities of causation we include the following discussion.

The probabilities of causation are important counterfactual queries that quantify the degree to which one event was a necessary or sufficient cause of another. 
Recently, variants on these have been used in medical diagnosis \cite{richens2020improving} to determine if a patient's symptoms would not have occurred had it not been for a specific disease. Here, the proposition binary variable $W$ is true is denoted $W=1$, and its negation, $W=0$, denotes the proposition $W$ is false.
\begin{enumerate}
    \item \textbf{Probability of necessity:} $$P(Y_{X=0}=0 \mid X=1, Y=1)$$
    The probability of necessity is the probability event $Y$ would not have
occurred without event $X$ occurring, given that $X,Y$ did in fact occur. 
    \item \textbf{Probability of sufficiency:} $$P(Y_{X=1}=1 \mid X=0, Y=0)$$
    The probability of sufficiency is the probability that in a situation where $X,Y$ were absent, intervening to make $X$ occur would have led to $Y$ occurring. 
    \item \textbf{Probability of necessity \& sufficiency:} $$P(Y_{X=0}=0, Y_{X=1}=1 \mid Z)$$
    The probability of necessity \& sufficiency quantifies the sufficiency and necessity of event $X$ to produce event $Y$ in context $Z$. As discussed in section~\ref{Section: two approaches to counterfactual inference}, joint counterfactual probabilities are well-defined.
\end{enumerate}

\subsection{Distinction between twin networks and abduction-action-prediction; Siamese Networks}

As was discussed in the main text, abduction-action-prediction counterfactual inference requires large computational resources. Twin networks were specifically designed to address this difficulty \cite{Pearl2009, balke1994counterfactual}. Indeed, consider the following passage from Pearl's ``Causality'' \cite[Section 7.1.4, page 214]{Pearl2009}:
\begin{quote}
``The advantages of delegating this computation [abduction] to inference in a Bayesian network [i.e., a twin network] are that the distribution need not be explicated, conditional independencies can be exploited, and local computation methods can be employed''
\end{quote}

This suggests that the computational resources required for counterfactual inference using twin networks can be less than in abduction-action-prediction. This was put to the test by \cite{grahamcopy} and shown empirically to be correct, with their abstract stating
\begin{quote}
``twin networks are faster and less memory intensive by orders of magnitude than standard [abduction-action-prediction] counterfactual inference'' 
\end{quote}

A key difference is that in a twin network, inference can be conducted in parallel rather than in the serial nature of abduction-action-prediction. For instance, sampling in twin networks is faster than in the abduction-action-prediction, as twin networks propagate samples simultaneously through the factual and counterfactual graphs---rather than needing to update, store and resample as in abduction-action-prediction. Thus, full counterfactual inference in a twin network can take up to no more than the amount of time sampling takes, while in abduction-action-prediction one incurs the additional cost of reusing samples and evaluating function values in the new mutilated graph. This is potentially advantageous for very large graphs, or for graphs with complex latent distributions that are expensive to sample.

Moreover, while twin networks appear to be conceptually similar to Siamese networks we highlight some key differences. Deep Twin Networks benefit from node merging in the intervention invariant variables while maintaining distinct information paths for the rest. As such there are parts of our network that are not completely shared between the factual and the counterfactual worlds. It is worth noting that \cite{reynaud2022d} implemented the Deep Twin Network as a Siamese network but had to resort to larger individual layers that could accommodate all the necessary information. Traditionally Siamese networks as the ones in  \cite{10.1007/978-3-030-58342-2_16} are used in classification tasks, where the inputs are items we wish to classify. In the case of the deep twin network the purpose is to regress the value of the causal outcome, with inputs signifying treatments. While architecturally the two approaches may coincide depending on implementation the purpose, design and use are distinct with the deep twin networks directly dependent on the causal graph of the phenomenon we investigate

\subsection{Description of Datasets Used }\label{appe_data_description}

In the German Credit Dataset the treatment is a four-valued variable corresponding to current account status, and the outcome is loan risk. The International Stroke Trial database was a large, randomized trial of antithrombotic therapy after stroke onset. The treatment is a three-valued variable corresponding to heparin dosage, and the outcome is a three-valued variable corresponding to different levels of patient recovery. In both cases we explore semi-synthetic settings, where the treatment and confounders are derived from the original dataset but the outcome is defined in a synthetic fashion. Synthetic outcomes allow us to determine the ground truth counterfactuals and probabilities of causation (see \ref{prob_of_causation_def} for definition).  
We also evaluate our algorithm with  real world outcome of the German Credit Dataset.

The Kenyan Water task is to understand whether protecting water springs in Kenya by installing pipes and concrete containers reduced childhood diarrhea, given confounders. First, monotonicity is a reasonable assumption here as protecting a spring is not expected to increase the bacterial concentration and hence increase the incidence of diarrhea. \cite{cuellar2020non} reported a low value for Probability of Necessity here---suggesting that children who developed diarrhea after being exposed to a high concentration of bacteria in their drinking water would have contracted the disease regardless. However, as there is no ground truth here, further studies reproducing this result with alternate methods are required to gain confidence in \cite{cuellar2020non}'s result. We follow the exact same data processing as in \cite{cuellar2020non} in order to ensure our results are comparable with literature,
The Kenyan water dataset originates from \cite{DVN28063_2015} lincenced under a non commercial use clause and with the requirement for secure storage, both conditions have been fulfilled by the authors. The data was preprocessed following  \cite{cuellar2020non}. 

For the Twin Mortality data, two versions were used. First databases provided by \cite{louizos2017causal} were processed to remove NaNs. No further processing was administered. This constituted the completely real version of the Twin Mortality dataset. However, as both \cite{louizos2017causal} and \cite{yoon2018ganite} process the their data to create a semi-synthetic task, in the spirit of proper comparison we used the data as processed and provided by \cite{yoon2018ganite}, with no additional processing. All datasets followed a train-test split of 70-30\%

\subsection{Unconfounded Synthetic example for RQ3}
\begin{equation} \label{equation: synthetic uncounfounded}
Y=
    \begin{cases}
      X  & \text{if}\ U_Y=0 \\
      0, & \text{if}\ U_Y=1 \\
      1, & \text{if}\ U_Y=2 \\
      
    \end{cases}
\end{equation}

Hence 
\begin{equation}
    P(N) = \frac{P(U_Y=0)}{P(U_Y=2)+P(U_Y=0)}
\end{equation}
\begin{equation}
    P(S) = \frac{P(U_Y=0)}{P(U_Y=1)+P(U_Y=0)}
\end{equation}
\begin{equation}
    P(NS) = P(U_Y=0)
\end{equation}

\subsection{Confounded Synthetic example for RQ3}
\begin{equation} \label{equation: synthetic counfounded rq1}
Y=
    \begin{cases}
      X \times Z, & \text{if}\ U_Y=0 \\
      0, & \text{if}\ U_Y=1 \\
      1, & \text{if}\ U_Y=2 \\
      
    \end{cases}
\end{equation}
where $$X:= U_x \oplus Z$$ 

Hence 
\begin{equation}
    P(N) = \frac{P(U_Y=0)P(Z=1)}{P(U_Y=2)+P(U_Y=0)P(Z=1)}
\end{equation}
\begin{equation}
    P(S) = \frac{P(U_Y=0)P(Z=1)}{P(U_Y=1)+P(U_Y=0)}
\end{equation}
\begin{equation}
    P(NS) = P(U_Y=0)P(Z=1)
\end{equation}

\subsection{Answering RQ3:}\label{appendix:rq1}

\subsubsection{Synthetic data experiments}
We test on both unconfounded and confounded causal models, with causal structure from Fig.\ref{examplescm} and Fig.\ref{figure: toy example} respectively.The data generation functions are defined in Equations~\ref{equation: synthetic counfounded rq1} and \ref{equation: synthetic uncounfounded} 
The functions remain monotonic in $X$. Given these, we construct synthetic datasets of $200000$ points split into training and testing under an $80-20$ split. The samples $U_y$ were drawn from either a uniform or a Gaussian distribution, depending on the experiment. Confounders $Z$ were taken from a uniform distribution. 
We opt for a high number of samples such that we do not bias our analysis due to small sample sizes. In the real world experiments the dataset sizes are smaller.  
Results for a trained twin network are in \ref{results_table_1}. 
We accurately estimate all Probabilities of Causation in both unconfounded and confounded cases when ground truth and candidate distributions are the same. In Figure 3 we also show performance of (a) unconfounded and (b) confounded cases as ground truth distribution of $U_Y$ in synthetic generating functions changes, but candidate training distributions remain fixed---showing robust estimation. 

\begin{figure*}[h]
        \centering

    \subfloat[
    ]{
        \includegraphics[width=0.45\linewidth]{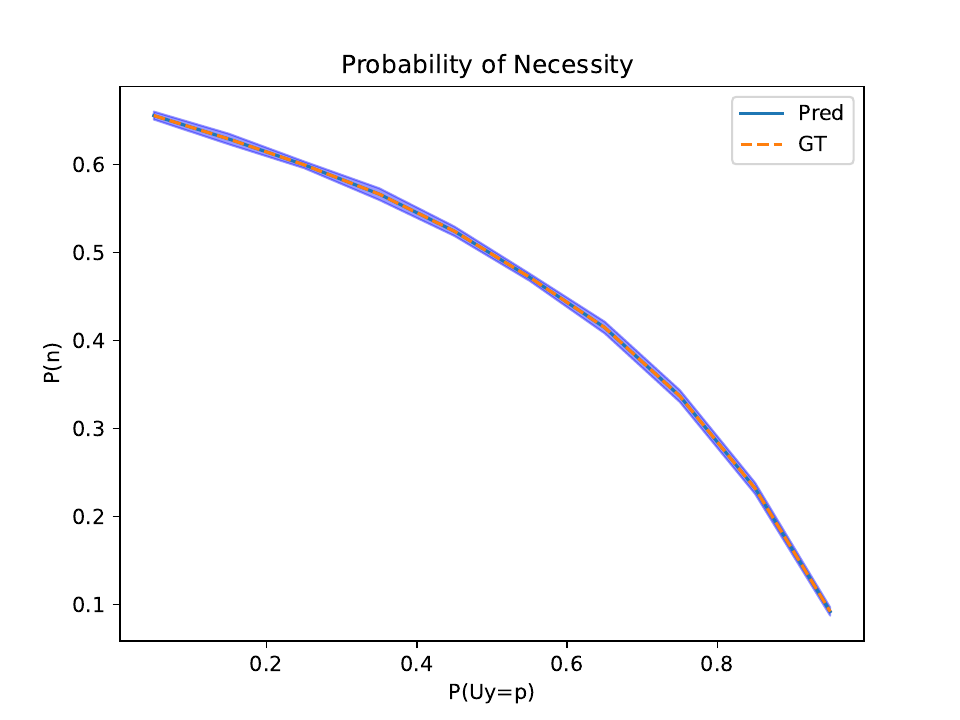}
        }
        \subfloat[
   ]{
        \includegraphics[width=0.45\linewidth]{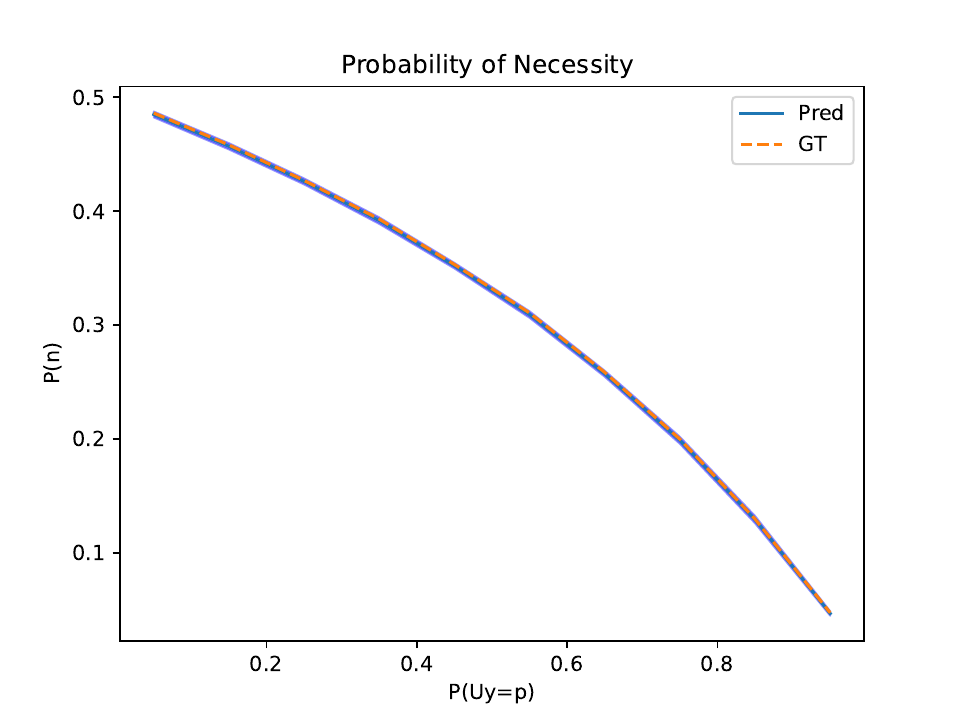}
        }
        
    \caption{Predicted \& ground truth Probability of Necessity as distribution of $U_Y$ varies in synthetic generating functions, but training distributions do not. Plots show robust estimation. (a) unconfounded, (b) confounded. Errors bars in both  \label{big figure: p_necessity_conf}}
   
\end{figure*}

\begin{table*}[]
\centering
\begin{tabular}{@{}lllll@{}}
\toprule
Method                                                             & $U_y$    & P(N)                                      & P(S)                                          & P(N\&S)                \\ \midrule
Synth Ground Truth                                                           & Uniform  & 0.5                                       & 0.5                                           & 0.33333                \\
Synth Twin Net          & Uniform  & $0.50214 \pm 0.00387 $                    & $ 0.50046 \pm 0.00631 $                       & $ 0.33449 \pm 0.00401$ \\
Synth w/ Conf Ground Truth & Gaussian        & 0.54706                         & 0.35512                                           & 0.27443                      \\
Synth w/ Conf Twin Net                                                 & Gaussian & $0.54563 \pm 0.00276$                     & $0.35177 \pm 0.00144 $                        & $0.27207 \pm 0.00125$  \\ \bottomrule
\end{tabular}

\caption{Results of Synthetic experiments. P(N): Prob. of Necessity; P(S): Prob. of Sufficiency; P(N\&S): Prob. of Necessity and Sufficiency. Our model achieves highly accurate estimations of the probabilities of causation on synthetic data.  \label{results_table_1}}
\end{table*}

\begin{table*}[]
\begin{minipage}[t]{1\textwidth}

\centering
\begin{tabular}{@{}lllll@{}}
\toprule
Method                                     & P(N)                   & P(S)                   & P(N\&S)   & AUC-ROC / F1           \\ \midrule
KW Median Child \textit{Cuellar et al. 2020}     & $0.12 \pm 0.01$        & -                      & -         &-             \\
KW TN Median Child                          & $0.13598 \pm 0.049$  & $0.09811 \pm 0.031 $ & $0.31778 \pm 0.012$ &- \\
KW TN Test Set                             & $0.06273 \pm 0.020 $ & $0.03914 \pm 0.016 $ & $0.08521 \pm 0.034 $&- \\ \midrule
Twin Mortality Ground Truth                         & 0.33372                & 0.01011                & 0.01353     & $0.83$/- \textit{Louizos et al. 2017}        \\
TM TN Test Set                             & $0.12241 \pm 0.019$  & $0.01401 \pm 0.003 $ & $0.01174 \pm 0.002$ & $0.86$/$0.83$ \\\bottomrule

\end{tabular}
\caption{Results of Kenyan Water (KW) \& Twins Mortality (TM) with Twin Network (TN), P(N): Prob. of Necessity; P(S): Prob. of Sufficiency; P(N\&S): Prob. of Necessity \& Sufficiency. In KW we agree \& improve on \cite{cuellar2020non}. In TM we overestimate P(N), but report accurate P(S) \& P(N\&S), \& better AUC than \cite{louizos2017causal}. \label{results_table_2}}
\end{minipage}
\end{table*}

\subsubsection{Real world data experiments}

\paragraph{\emph{Kenyan Water dataset:}}
The Kenyan Water task is to understand whether protecting water springs in Kenya by installing pipes and concrete containers reduced childhood diarrhea. First, monotonicity is reasonable here as protecting a spring is not expected to increase the bacterial concentration and hence increase the diarrhea incidence. \cite{cuellar2020non} reported a low value for Probability of Necessity here---suggesting that children who developed diarrhea after being exposed to a high concentration of bacteria in their drinking water would have contracted the disease regardless. However, as there is no ground truth, further studies reproducing this result with alternate methods are required to gain confidence in \cite{cuellar2020non}'s result. We follow the same data processing as in \cite{cuellar2020non}, detailed in the Appendix. 
Our findings are in Table~\ref{results_table_2} of the Appendix and agree with \cite{cuellar2020non} on Probability of Necessity. Moreover, unlike \cite{cuellar2020non}, we can also compute Probability of Sufficiency and Probability of Necessity and Sufficiency. We can thus offer a more comprehensive understanding of the role protecting water springs plays in childhood disease. Our results show that exposure to water-based bacteria is not a necessary condition to exhibit diarrhea and it is neither a sufficient, nor a necessary-and-sufficient condition. This provides further evidence that protecting water springs has little effect on the development of diarrhea in children in these populations, indicating the source of the disease is not related to water. 

\subsection{Determining monotonicity direction}\label{Appendix:rq2 results}

In Table~\ref{current_acount_monot_a} we provide the ATE of a semi-synthetic existing account status from the German Credit Dataset \cite{Dua:2019} with a fully synthetic outcome defined in \ref{synthetic_outcome_def}. We observe that for a control treatment $0$ changing the treatment to $[1,2]$ the ATE increases, indicating a monotonic increasing relationship. In addition, we could observe the interventional probabilities where as we increase the value of the treatment the probability of a higher outcome increases, we show this in the Appendix's Table~\ref{current_acount_monot_b}. This reinforces our beliefs regarding the type of monotonicity. 
 \begin{table}[t]
    \centering
    \resizebox{.6\columnwidth}{!}{
   \begin{tabular}{@{}llll@{}}
\toprule
\textbf{ATE} & \textbf{0} & \textbf{1} & \textbf{2} \\ \midrule
\textbf{0}   & 0          & 0.3059     & 0.8914     \\
\textbf{1}   & -0.3059    & 0          & 0.5854     \\
\textbf{2}   & -0.8914    & -0.5854    & 0          \\ \bottomrule
\end{tabular}}
   \caption{ Treatment: Semi-Synthetic Existing account status, Outcome: Synthethic. As the change from $0$ to $1 \& 2$ has a positive ATE, the relationship is increasing monotonic}
    \label{current_acount_monot_a}
\end{table}

 Similarly, Tables~\ref{existing_acount_monot_risk_a_ate},~\ref{existing_acount_monot_risk_b_py_do_x} include the ATE and the interventional probabilities for a real world variant of the above dataset where the treatments are again the current account status of the individual but the outcome is their classification as good or bad risk.  At this point, we may call upon our domain knowledge and determine if the break in the monotonic trend is due to noisy observations or a different ordering of the treatments. As this is a real world dataset in which the outcome attribution is inherently noisy we observe an outlier behavior from treatment $1$. Upon closer inspection we observe that treatment $1$ corresponds to a negative balance in the individuals checking account while treatment $0$ indicates no existing checking account. As such one could either switch the treatment ordering to obey the monotonicity, or in the case that this break in monotonicity is suspected to be due to noisy data one could enforce prior knowledge-based monotonicity. Here, we follow our prior knowledge and attribute the break of monotonicity to noise. 
 Our reasoning is based on the fact that an individual without a prior credit account is a larger unknown for a financial institution.

\begin{table}[t]
    \centering
    \resizebox{.65\columnwidth}{!}{
   \begin{tabular}{llll}
\hline
\textbf{$P(Y|do(X))$} & \textbf{0.0} & \textbf{1.0} & \textbf{2.0} \\ \hline
\textbf{0}         & 0.6396       & 0.2056       & 0.1548       \\
\textbf{1}         & 0.4635       & 0.2518       & 0.2847       \\
\textbf{2}         & 0.1656       & 0.2620       & 0.5723       \\ \hline
\end{tabular}}
    \caption{Same data as Table~\ref{current_acount_monot_a}. Rows are treatments, and columns are outcomes}
    \label{current_acount_monot_b}
\end{table}

\begin{table}[]
    \centering
   \begin{tabular}{@{}lllll@{}}
\toprule
\textbf{ATE} & \textbf{0} & \textbf{1} & \textbf{2} & \textbf{3} \\ \midrule
\textbf{0}         & 0          & -0.0791    & 0.1029     & 0.2174     \\
\textbf{1}         & 0.0791     & 0          & 0.1820     & 0.2965     \\
\textbf{2}         & -0.1029    & -0.1820    & 0          & 0.1145     \\
\textbf{3}         & -0.2174    & -0.2965    & -0.1145    & 0          \\ \bottomrule
\end{tabular}
    \caption{Treatment: Account status, Outcome: Risk Status}
    \label{existing_acount_monot_risk_a_ate}
\end{table}

\begin{table}[]
    \centering
   \begin{tabular}{@{}lll@{}}
\toprule
\textbf{$P(Y|do(X))$} & \textbf{0.0} & \textbf{1.0} \\ \midrule
\textbf{0}         & 0.1919       & 0.8081       \\
\textbf{1}         & 0.5450       & 0.4549       \\
\textbf{2}         & 0.3336       & 0.6663       \\
\textbf{3}         & 0.1265       & 0.8735       \\ \bottomrule
\end{tabular}
    \caption{ $P(Y|do(X))$ of the same dataset, rows indicate treatments while columns outcomes}
    \label{existing_acount_monot_risk_b_py_do_x}
\end{table}

\subsection{Answering RQ1, RQ2}\label{appendix:unconstrained_ps}

\subsubsection{Synthetic Outcome For German Credit Score data} \label{synthetic_outcome_def}
\begin{equation} \label{equation: synthetic counfounded}
Y=
    \begin{cases}
      X + Z  & \text{if}\ U_Y=0 \\
      0, & \text{if}\ U_Y=1 \\
      X * Z, & \text{if}\ U_Y=2 \\
           2, & \text{if}\ U_Y=3 \\
                  1, & \text{if}\ U_Y=4 \\
                        \text{step}(X-1), & \text{if}\ U_Y=5 \\
                              2*\text{step}(X-1), & \text{if}\ U_Y=6 \\
    \end{cases}
\end{equation}
where the treatment $X$ and the confounders $Z$ span the range $X,Z \in [0,2]$. Step is the Heaviside step function. In our experimentation $U_Y$ was drawn from a uniform distribution

\subsubsection{Synthetic Outcome for Internatinal Stroke Trial}\label{syntehtic_aspirin_def}
For $X$ the dosage of aspirin treatment, as detailed in the IST Dataset, and counfouders SEX $:=$ biological sex of patient, AGE$:=$ age of patient thresholded at 71 years, CONSC~$:=$ level of consciousness the patient arrived in hospital with. $Y =1/(1+e^{-g})$ with $g$ being given by:

\begin{equation}
\begin{aligned}
     g = X + \text{SEX} + &0.2*(\text{CONSC}-1) \\
     &+ 0.5*X * \text{SEX}*\text{AGE}  + U_y
\end{aligned}
\end{equation}

 \begin{table}[t]
\centering
\resizebox{1\columnwidth}{!}{
\begin{tabular}{rllll}
\toprule
   P & 0.0                  & 1.0                  & 2.0                  & 3.0                  \\
\midrule
    0 & 0                    & $0.4773 \pm 0.0182 $ & $0.4243 \pm 0.0272 $ & $0.3894 \pm 0.0147 $ \\
    1 & $0.6049 \pm 0.0386 $ & 0                    & $0.5791 \pm 0.0340 $ & $0.5616 \pm 0.0183 $ \\
    2 & $0.6081 \pm 0.0388 $ & $0.6025 \pm 0.0806 $ & 0                    & $0.5221 \pm 0.0130 $ \\
    3 & $0.6265 \pm 0.0297 $ & $0.6580 \pm 0.0431 $ & $0.5188 \pm 0.0272 $ & 0                    \\
\bottomrule
\end{tabular}}
\caption{Switched counterfactual ordering -- Probability of counterfactual  $P(T,T')=P(Y_{X=T'}=1 \mid X=T, Y=0)$  -- columns and rows are Treatments -- We observe counter-intuitive probabilities of necessity as the lower triangular sub-matrix has higher probabilities than the upper triangular \label{weird_order_p_c}}
\end{table}

 \begin{table}[t]
\centering
\resizebox{1\columnwidth}{!}{
\begin{tabular}{rlll}
\toprule
   P                 & 1.0                  & 2.0                  \\
\midrule
        0 & 0                    & $0.1260 \pm 0.0070 $ & $0.0000 \pm 0.0000 $ \\
        1 & $0.0000 \pm 0.0000 $ & 0                    & $0.2262 \pm 0.0429 $ \\
        2 & $0.0000 \pm 0.0000 $ & $0.0000 \pm 0.0000 $ & 0                    \\
\bottomrule
\end{tabular}}
\caption{$P=P(Y_{X=1}=Column|Y_{X=0}=Row)$.}
\label{account-synthetic-probs-of-nec}
\end{table}
 
\begin{table}[t]
\centering
\begin{tabular}{rrrr}
\toprule
   P &    0.0 &    1.0 &    2.0 \\
\midrule
               0 & 0.0000 & 0.1190 & 0.0079 \\
               1 & 0.0000 & 0.0000 & 0.2037 \\
               2 & 0.0000 & 0.0000 & 0.0000 \\
\bottomrule
\end{tabular}
\caption{$P=P(Y_{X=1}=Column|Y_{X=0}=Row)$.}
\label{account-synthetic-probs-of-nec-gt}
\end{table}

\

\begin{table}[t]
\centering
\resizebox{1\columnwidth}{!}{
\begin{tabular}{rlll}
\toprule
   P                  & 1.0                  & 2.0                  \\
\midrule
        0 & 0                    & $0.0482 \pm 0.0006 $ & $0.1840 \pm 0.0001 $ \\
        1 & $0.0011 \pm 0.0019 $ & 0                    & $0.1130 \pm 0.0069 $ \\
        2 & $0.0000 \pm 0.0000 $ & $0.0135 \pm 0.0058 $ & 0                    \\
\bottomrule
\end{tabular}}
\caption{ $P=P(Y_{X=1}=Column|Y_{X=0}=Row)$.}
\label{heparin-synthetic-preds}
\end{table}

\begin{table}[t]
\centering
\begin{tabular}{rrrr}
\toprule
   P &    0.0 &    1.0 &    2.0 \\
\midrule
               0 & 0.0000 & 0.0266 & 0.0917 \\
               1 & 0.0000 & 0.0000 & 0.0911 \\
               2 & 0.0000 & 0.0000 & 0.0000 \\
\bottomrule
\end{tabular}
\caption{$P=P(Y_{X=1}=Column|Y_{X=0}=Row)$.}
\label{heparin-synthetic-preds-gt}
\end{table}

\subsubsection{Treatment: Existing Account Status , Outcome: Risk, switched ordering}
In Table \ref{weird_order_p_c} we show switching the ordering of treatment 0 and 1 leads to non-intuitive results akin to no constraints. 

\subsubsection{Treatment: Semi-Synthetic Account Status, Outcome: Synthetic }
 Tables \ref{account-synthetic-probs-of-nec},\ref{account-synthetic-probs-of-nec-gt} show counterfactual probabilities from our method applied to the semi-synthetic account status treatment and synthetic outcome. We note that our model slightly violates the monotonicity constraints by providing non zero probabilities to two cases where they should be 0. However, both of these are within our acceptable experimental error, with one being less than $1\%$, the other being just over $1\%$

 \subsubsection{Treatment: Heparin, Outcome: Synthetic}
  Tables \ref{heparin-synthetic-preds},\ref{heparin-synthetic-preds-gt} show the counterfactual probabilities for the semi-synthetic heparin treatment and synthetic outcome.

\end{document}